\documentclass[11pt]{article}
\usepackage[margin=1in]{geometry}
\usepackage{times}
\usepackage{amsmath, amssymb, amsthm}
\newtheorem{theorem}{Theorem}
\newtheorem{remark}{Remark}
\usepackage{graphicx}
\usepackage{booktabs}
\usepackage{enumitem}
\usepackage{hyperref}
\usepackage{mathrsfs}
\usepackage{cite}
\usepackage{float}
\usepackage{tikz}
\usepackage[ruled,vlined]{algorithm2e}
\usepackage{natbib}

\title{Variance-Bounded Evaluation of Entity-Centric AI Systems Without Ground Truth: Theory and Measurement}
\author{
  Kaihua Ding \\
  University of Pennsylvania \\
  \texttt{dkaihua@upenn.edu}
}
\date{\today}

\begin{document}

\maketitle

\begin{abstract}
Reliable evaluation of AI systems remains a fundamental challenge when ground truth labels are unavailable, particularly for systems generating natural language outputs like AI chat and agent systems. Many of these AI agents and systems focus on entity-centric tasks. In enterprise contexts, organizations deploy AI systems for entity linking, data integration, and information retrieval where verification against gold standards is often infeasible due to proprietary data constraints. Academic deployments face similar challenges when evaluating AI systems on specialized datasets with ambiguous criteria. Conventional evaluation frameworks, rooted in supervised learning paradigms, fail in such scenarios where single correct answers cannot be defined. We introduce VB-Score, a variance-bounded evaluation framework for entity-centric AI systems that operates without ground truth by jointly measuring effectiveness and robustness. Given system inputs, VB-Score enumerates plausible interpretations through constraint relaxation and Monte Carlo sampling, assigning probabilities that reflect their likelihood. It then evaluates system outputs by their expected success across interpretations, penalized by variance to assess robustness of the system. We provide formal theoretical analysis establishing key properties—including range, monotonicity, and stability—along with concentration bounds for Monte Carlo estimation. Through case studies on AI systems with ambiguous inputs, we demonstrate that VB-Score reveals robustness differences hidden by conventional evaluation frameworks, offering a principled measurement framework for assessing AI system reliability in label-scarce domains.

\end{abstract}

\section{Introduction}
\label{s:intro}
Evaluating AI systems that generate natural language outputs—such as chat or agent models—poses fundamental measurement challenges when ground truth labels are unavailable, costly to obtain, or unreliable. In practice, many enterprise and research applications, including entity linking, data integration, and information retrieval, operate under conditions where gold-standard verification is infeasible due to proprietary data, limited annotation budgets, specialized domain expertise requirements, or the subjective nature of outputs.

Consider AI systems deployed for enterprise data integration, where organizations frequently integrate data acquired from multiple vendors—pseudonymized or anonymized entity records spanning user transactions, browsing histories, and operational logs. AI systems must reconcile heterogeneous sources with ambiguous and inconsistent schemas, yet verification of the resulting data products is frequently infeasible even after extensive processing. Similarly, AI chat systems deployed for customer service must handle ambiguous user queries where multiple valid interpretations exist, but obtaining ground truth labels for every possible user intent is impractical.

In academic contexts, AI systems are increasingly deployed for specialized tasks such as linking research abstracts to publications, analyzing scientific literature, or processing restricted datasets. These applications expose fundamental evaluation challenges: linked entity names may be shared by multiple individuals; dataset references may correspond to different domains or versions, and literature might use identical abbreviations but with field-dependent interpretations. Most critically, unless validated through extensive manual verification, the accuracy of AI system outputs for such tasks, especially cutting-edge academic research text, remains uncertain.

These examples illustrate that AI system evaluation is both critical and challenging, yet often lacks reliable ground truth. Users typically interact with AI systems through natural language queries that may be ambiguous (e.g., ``electronic health records Dr. John Smith'') or contain partially incorrect details (e.g., wrong employer or year). When such queries are processed by AI chat or agent systems, conventional evaluation frameworks—which assume a single correct answer—become ill-posed. In practice, even human assessors may be unable to specify unique ground truth, and users themselves may be uncertain of their intended meaning.

Several classic query response and information retrieval frameworks exist for natural language text evaluation. The classic Cranfield paradigm \citep{cleverdon1967cranfield}, which underpins modern information retrieval evaluation, relies on expert-labeled relevance judgments—an approach that is costly and impractical in domains requiring specialized or proprietary knowledge. Existing Named Entity Linking (NEL) and Named Entity Recognition (NER) benchmarks (e.g., ACE, TAC KBP, CoNLL) \citep{tjong2003conll} evaluate precision, recall, and F1 under strict supervision, but they fail to capture \emph{real-world} tasks where criteria are ambiguous, incomplete, or undefined. More recently, some AI systems are evaluated using human annotation-based Elo ratings~\cite{chiang2024chatbot}, which, while popular, are also expensive and difficult to scale. Recent work has explored automating evaluation with large language models (LLMs) \citep{fabbri2021summeval}, but these methods remain fragile: LLMs may fail on tasks without ground truth and may never have seen restricted or rare datasets during training.

This paper introduces a variance-bounded evaluation framework for entity-centric AI systems (VB-Score), where both the prompt and response contain entities—a common scenario in both industry and academic applications. Instead of assuming one ground truth, we enumerate a set of plausible interpretations $\mathcal{I}(Q)=\{I_1,\ldots,I_n\}$ for a system input prompt $Q$, assign probabilities $P(I_i\mid Q)$ reflecting plausibility, and evaluate the AI system's output by expected success across interpretations. We further quantify \emph{robustness} through a variance term that penalizes systems that perform well only on a narrow subset of plausible intents, thereby rewarding consistent performance across diverse scenarios.

Our contributions are:
\begin{itemize}[leftmargin=1.25em]
  \item A new problem formulation for evaluating entity-centric AI systems when output criteria are incomplete\textbf{/}ambiguous and ground truth is unavailable.
  \item \emph{VB-Score}: an unsupervised, normalized metric that computes expected success across plausible interpretations and includes a variance penalty to measure robustness.
  \item Formal theoretical analysis of VB-Score properties, including range, monotonicity, and stability.
  \item Case studies on entity-centric tasks demonstrating how VB-Score reveals AI system robustness, a metric not provided by conventional ground-truth-focused evaluation methodologies.
\end{itemize}

\section{Related Work}
\label{s:related_work}

\paragraph{Evaluation Without Ground Truth.}
The challenge of evaluating systems that generate natural language text without ground truth has been explored across multiple domains. \citet{zafarani2015evaluation} proposed methods for assessing models in social media research where labeled data are unavailable. Recent work has examined model explanations \citep{rawal2025evaluating}, clinical AI systems under uncertain ground truth \citep{kiyasseh2024evaluating,stutz2023evaluating}, and entity disambiguation \citep{nanayakkara2025unsupervised,ji2011knowledge}. Our work extends these ideas to evaluating AI chat system responses using constraint relaxation and Monte Carlo sampling.

\paragraph{Robustness in Evaluation.}
The concept of robustness has been studied extensively across domains. \citet{herman2015robustness} define robustness in water systems planning as the ability to perform well under diverse future conditions. \citet{parker2018robustness} associate robustness with stability of decision-making competence over time. In machine learning, robustness is often studied in adversarial contexts, focusing on bias-variance trade-offs \citep{wu2022adversarial}. Our work contributes a principled approach to measuring robustness in entity-centric AI systems under input ambiguity, where variance in performance across plausible interpretations serves as a proxy for system reliability.

\paragraph{Diversified Information Retrieval.}
Our evaluation framework is conceptually related to diversified information retrieval, which aims to present users with results that capture multiple facets of their information needs \citep{clarke2008novelty,sakai2011perintent,agrawal2009diversifying}. These methods rely on intent-aware metrics that evaluate how well a system satisfies distinct user intents. Similarly, entity-centric AI systems must reason over diverse possible interpretations of a query or instruction. VB-Score generalizes this to settings without explicit ground truth, where prompt intents are inferred from input ambiguity rather than predefined labels.

\paragraph{Measurement Foundations.}
Finally, our work aligns with the SIGMETRICS tradition of systematic measurement and rigorous analysis \citep{hodge1982workload,clark1979feature,frachtenberg2022multifactor}. By providing formal theoretical properties (range, monotonicity, stability) and statistically valid confidence intervals, our framework contributes to developing more robust and reliable evaluation methodologies for entity-centric information systems.

\section{The Case for Variance-Based Evaluation}
\label{s:variance_case}

\subsection{What We Always Have: Inputs and Outputs}

Even when ground truth labels are unavailable, AI systems—whether chat interfaces or agent workflows—possess two fundamental, observable components. On the \emph{input side}, there is the user prompt or instruction, containing text, intent signals, and contextual cues. On the \emph{output side}, there is the system's response: generated text, retrieved documents, or task-specific actions. While we may lack definitive gold-standard labels for correctness, these input-output pairs define observable distributions that can be systematically measured and analyzed.

For the input side, we can quantify uncertainty through probability distributions over plausible interpretations, measuring the degree of ambiguity inherent in user queries. For the output side, we can characterize the distribution of system responses through Monte Carlo sampling, capturing variability across multiple runs. Although direct supervised comparison against ground truth is infeasible, analyzing the statistical relationship between input variability and output consistency enables a principled, variance-based evaluation framework.

\subsection{Statistical Foundations}

Our approach draws inspiration from classical statistical inference. In statistics, population characteristics can be estimated without exhaustive enumeration through carefully designed sampling procedures and distributional analysis. We adopt an analogous perspective: treating the space of plausible input interpretations as one population and the space of system outputs as another. By systematically varying the input distribution—for example, by enumerating plausible interpretations of an ambiguous query—we observe corresponding variations in the output distribution. Repeated trials of this process, ideally randomized to avoid systematic bias, allow us to estimate the stability and robustness of system performance.

This motivates our variance-based evaluation framework. Rather than comparing outputs against fixed gold labels (which may not exist), we evaluate systems by characterizing the relationship between input variability and output robustness. A system that performs consistently well across diverse plausible interpretations demonstrates reliability; a system whose performance varies widely across interpretations reveals brittleness. By penalizing variance in performance, our framework rewards systems that are robust to input ambiguity—a critical property for real-world deployment where user intents are often uncertain or underspecified.

\section{Framework}
\label{s:framework}

We call this the \emph{variance-bounded evaluation framework} because it evaluates system performance under intent uncertainty using both the expected success (mean) and its variability (variance). The VB-Score measures the average probability of satisfying a user intent, while the variance penalty bounds this score by penalizing inconsistency across all plausible interpretations.

\subsection{Problem Setup and Notation}

Let $Q$ denote a prompt or system instruction. We specifically focus on queries $Q$ that admit various responses; when $Q$ has a deterministic response or a certain gold label, the evaluation task becomes trivial and no robustness measurement of the system response is needed. Because $Q$ may be ambiguous, underspecified, or partially incorrect, we assume there exists a \emph{set of plausible interpretations} $\mathcal{E}(Q)=\{E_1,\dots,E_n\}$ with a probability vector $\boldsymbol{\pi}(Q)=(\pi_1,\dots,\pi_n)$, where $\pi_i \equiv P(E_i\mid Q)$ and $\sum_i \pi_i=1$. An AI system returns a ranked list $S@k=[d_1,\dots,d_k]$ of responses. We write $\mathrm{rel}(d,E)\in\{0,1\}$ for whether result $d$ is relevant to entity $E$ (e.g., the page \emph{about} $E$, or a document primarily describing $E$).

We conceive of two \emph{observable populations}:
(i) the \textbf{input population} of queries and their intent distributions $\boldsymbol{\pi}(Q)$; and
(ii) the \textbf{output population} of ranked results and their entity assignments $\phi(d)\in\mathcal{E}(Q)$ (obtained via open-world LLM-based entity linking).
Even without gold labels, these two populations admit stable descriptive and inferential statistics.

\subsection{Input-Side: Candidate Distribution}
\label{s:input}

This stage refines the query into a distribution over plausible entities. We construct the candidate set $\mathcal{E}(Q)$ and its probability distribution $\boldsymbol{\pi}(Q)$ in three steps:

\paragraph{(1) Linking \& Scoring to Generate Candidates.} Apply an entity linker or knowledge base-backed candidate generator to $Q$ to produce candidates $\{(E_i, s_i)\}_{i=1}^n$, where each $E_i$ is a candidate entity and $s_i$ is a score. Convert scores $\{s_i\}$ to a probability vector $\boldsymbol{\pi}$ using temperature-scaled softmax:
\[
\pi_i \propto \exp(s_i/T),\quad \sum_i \pi_i=1.
\]
We fix $T=1$ to remain consistent with our label-free evaluation setting.

\paragraph{(2) Constraint Relaxation.} 
When $Q$ specifies attributes that may not all be exactly matched in the knowledge base (KB), we evaluate entities by the \emph{maximally satisfiable subset} of constraints. Let $C=\{c_j\}_{j=1}^m$ be the set of query constraints with weights $w_j \ge 0$. For each candidate entity $E$, define a violation indicator:
\[
\mathbf{1}[\neg c_j(E)] =
\begin{cases}
1, & \text{if $E$ violates constraint $c_j$}, \\
0, & \text{if $E$ satisfies $c_j$}.
\end{cases}
\]
The total violation penalty is:
\[
\Delta(E) = \sum_{j=1}^m w_j \,\mathbf{1}[\neg c_j(E)].
\]
We normalize across the candidate set to obtain a probability distribution:
\[
\pi(E) = \frac{\exp(-\Delta(E))}{\sum_{E'\in \mathcal{E}(Q)} \exp(-\Delta(E'))}.
\]

\paragraph{(3) Ambiguity Coverage \& Deduplication.}
We preserve multiple plausible interpretations but remove negligible and duplicate candidates using explicit rules: truncation (retain candidates using a fixed threshold, top-$K$, or cumulative-mass cutoff), and deduplication (canonicalize candidates through KB identifier mapping, string normalization, and semantic clustering).

\subsection{Output-Side: Tagging and Per-Intent Gains}
\label{s:output}

This stage assesses whether retrieved results cover the plausible entity interpretations. Each retrieved item $d_j$ is re-linked to an entity $\phi(d_j)\in\mathcal{E}(Q)$ using snippets, titles, or landing pages.

Define a per-intent \emph{gain} at cutoff $k$ as:
\[
g_i(S@k) = \max_{1 \le j \le k} \mathbf{1}\{\phi(d_j) = E_i\},
\]
which equals $1$ if at least one result in the top $k$ is about entity $E_i$, and $0$ otherwise. A rank-sensitive variant weights matches by their rank position using discounted cumulative gain (DCG).

\subsection{Variance-Bounded Metric}
\label{s:metric}

Given $(\mathcal{E}(Q),\boldsymbol{\pi}(Q))$ and gains $\{g_i\}$, we define the \emph{expected success} at cutoff $k$:
\[
\mathrm{ES}(Q, S@k) = \sum_{i=1}^n \pi_i(Q)\, g_i(S@k) \in [0,1].
\]
In the binary-gain case $g_i\in\{0,1\}$, this equals the probability that a randomly drawn intent $E_i\sim\boldsymbol{\pi}(Q)$ finds at least one relevant item in the top-$k$.

To incorporate robustness across intents, let $X$ be the Bernoulli success indicator with $\mathbb{E}[X]=\mathrm{ES}(Q,S@k)$. Its variance is:
\[
\mathrm{Var}(X)=\mathrm{ES}(Q,S@k)\bigl(1-\mathrm{ES}(Q,S@k)\bigr).
\]
We define the \textbf{Variance-Bounded Score (VB-Score)}:
\[
\mathrm{VB}_\alpha(Q,S@k)=\mathrm{ES}(Q,S@k)-\alpha\,\sqrt{\mathrm{Var}(X)},\quad \alpha\ge0,
\]
which lies in $[0,1]$ and favors systems that perform consistently across plausible intents.

\subsection{Estimating VB and Uncertainty in Practice}
\label{s:estimation}

In the absence of ground truth, two main sources of uncertainty must be addressed:  
(i) estimation of the intent distribution $\boldsymbol{\pi}(Q)$; and  
(ii) variability induced by paraphrasing, constraint relaxation, and stochasticity in entity linking.  
To quantify these, we adopt a Monte Carlo procedure with $B$ replicas (Algorithm~\ref{alg:vb}).  
Each replica perturbs the input side (query interpretations) and re-tags the system output, yielding a distribution of variance-bounded scores.

Formally, the expected success for prompt/query $Q$ is estimated as
\[
\widehat{\mathrm{ES}}(Q,S@k) \;=\; \frac{1}{B}\sum_{b=1}^{B} \sum_{i=1}^{n_b} \pi^{(b)}_i(Q)\, g^{(b)}_i(S@k),
\]
where replica $b$ produces a candidate set $\mathcal{E}^{(b)}(Q)$, intent probabilities $\boldsymbol{\pi}^{(b)}(Q)$, and re-tagged gains $g^{(b)}_i(S@k)$.  

A nonparametric bootstrap across the $B$ replica scores provides confidence intervals:
\[
\text{CI}_{1-\delta} \;=\; \Bigl[\,\widehat{\mathrm{ES}} \;-\; z_{1-\delta/2}\,\tfrac{\widehat{\sigma}}{\sqrt{B}},\;\;
\widehat{\mathrm{ES}} \;+\; z_{1-\delta/2}\,\tfrac{\widehat{\sigma}}{\sqrt{B}}\,\Bigr],
\]
where $\widehat{\sigma}^2$ is the sample variance of replica scores and $\text{CI}_{1-\delta}$ is the $(1-\delta)$ confidence interval.  

At the \emph{collection level}, with query set $\mathcal{Q}$, we report macro-averaged results:
\[
\mathrm{VB}@k(S) \;=\; \frac{1}{|\mathcal{Q}|}\sum_{Q\in\mathcal{Q}} \widehat{\mathrm{VB}}(Q,S@k),
\]
with CIs obtained by resampling queries. If a small development set with partial labels exists, $\boldsymbol{\pi}(Q)$ can be calibrated (e.g., Platt scaling, isotonic regression), and tagger precision for $\phi(d)$ validated.  
Otherwise, parameters such as the temperature $T$ or constraint weights should be treated as sensitivity knobs, and results reported across a small range of values.

\begin{algorithm}[H]
\SetAlgoLined
\KwIn{Query $Q$, retrieval system $\mathcal{S}$, cutoff $k$, number of replicas $B$}
\KwOut{Estimated VB-Score $\widehat{\mathrm{VB}}(Q,S@k)$ with confidence intervals}
\For{$b \gets 1$ \KwTo $B$}{
  \tcp{Input-side: candidate generation}
  Generate $\mathcal{E}^{(b)}(Q)$ via linking, constraint relaxation, and ambiguity coverage\;
  Compute probability distribution $\boldsymbol{\pi}^{(b)}(Q)$\;

  \tcp{System run and output tagging}
  Run system $\mathcal{S}$ on $Q$ to obtain $S@k$\;
  Tag each $d_j \in S@k$ with entity $\phi^{(b)}(d_j)\in\mathcal{E}^{(b)}(Q)$\;

  \tcp{Replica scoring}
  Compute per-intent gains $g^{(b)}_i(S@k)$\;
  Compute replica score $\mathrm{VB}^{(b)}(Q,S@k)$\;
}
\textbf{Aggregation:} average replica scores and compute bootstrap confidence intervals\;
\[
\widehat{\mathrm{VB}}(Q,S@k) \;=\; \frac{1}{B}\sum_{b=1}^B \mathrm{VB}^{(b)}(Q,S@k).
\]
\caption{Monte Carlo estimation of variance-bounded evaluation for a single prompt/query.}
\label{alg:vb}
\end{algorithm}

The algorithm above formalizes how replicas are generated and aggregated. It emphasizes that robustness is not inferred from a single run but from a distribution of perturbed interpretations. In this way, VB evaluation parallels established resampling methods in statistics, ensuring stability even without ground truth labels.

\subsection{Flowchart Summary}
To complement the algorithmic description, Figure~\ref{fig:vb-nelir-flow-min} depicts the entire framework as a sequential pipeline. The process begins with query metadata, proceeds through candidate generation and intent probability assignment (A), continues with retrieval and tagging (B), evaluates with ES and VB metrics (C), and concludes with Monte Carlo replicas and bootstrap aggregation (D). 

Each stage of the flowchart corresponds directly to a subsection above: 
\begin{itemize}
    \item Block (A) illustrates candidate enumeration, constraint relaxation, and ambiguity handling. 
    \item Block (B) shows how retrieved results are aligned with candidate intents to compute per-intent gains. 
    \item Block (C) captures the transition from gains to ES and VB-Scores, highlighting the role of robustness penalties. 
    \item Block (D) illustrates uncertainty quantification and aggregation into collection-level results.  
\end{itemize}

This sequential diagram underscores that the VB-NEL-IR framework is both modular and reproducible: input interpretation, output tagging, metric computation, and uncertainty aggregation can each be validated and refined independently.

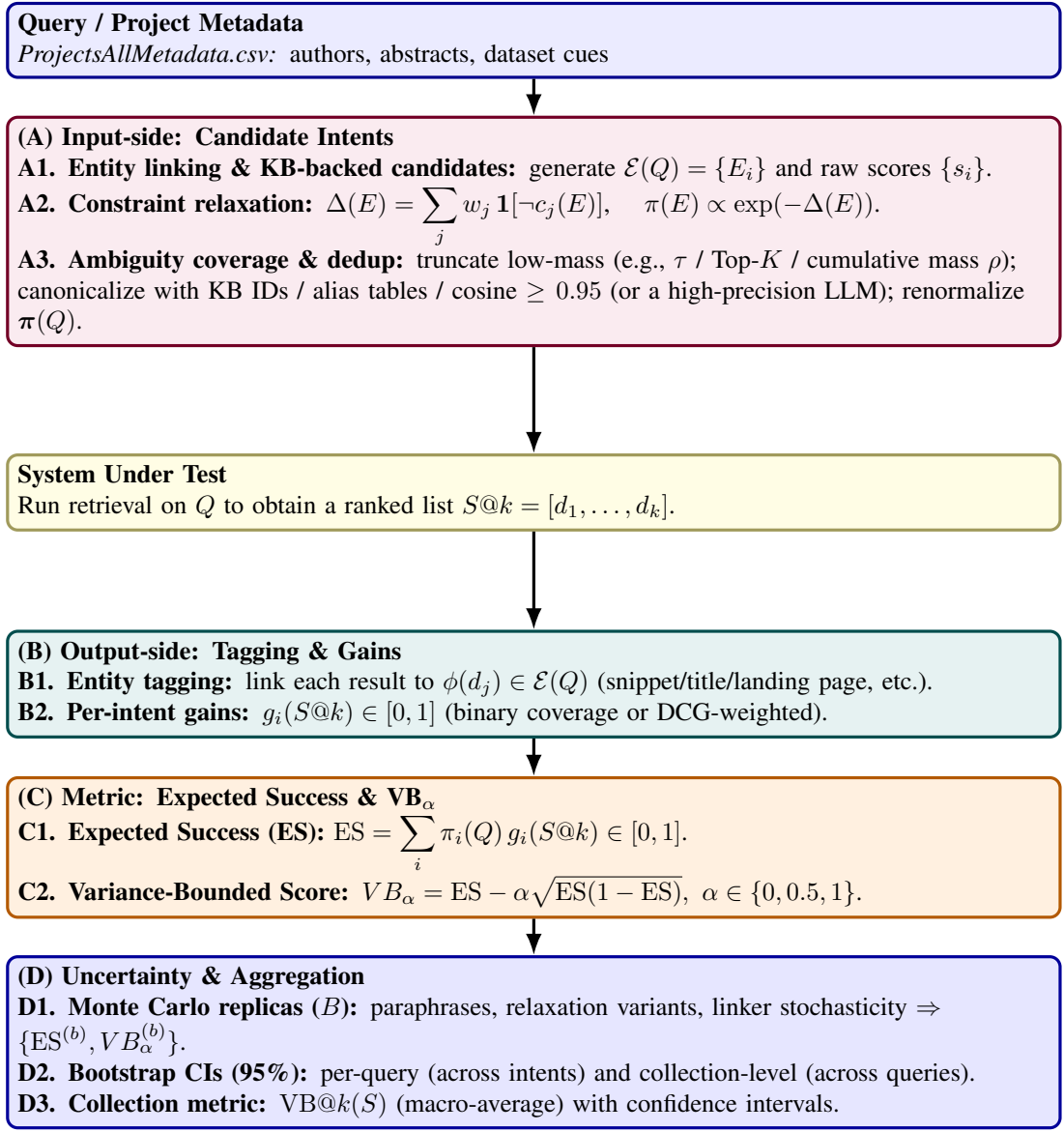
\begin{figure}[H]
\centering
\usetikzlibrary{arrows.meta,positioning}
\begin{tikzpicture}[
  font=\small,
  >=Latex,
  node distance=6mm,
  box/.style={
    draw, very thick, rounded corners, align=left,
    inner sep=4pt, text width=0.86\linewidth
  },
  head/.style={box, fill=blue!8, draw=blue!55!black},
  Ablk/.style={box, fill=purple!8, draw=purple!60!black},
  Sblk/.style={box, fill=yellow!12, draw=yellow!55!black},
  Bblk/.style={box, fill=teal!10, draw=teal!60!black},
  Cblk/.style={box, fill=orange!12, draw=orange!70!black},
  Dblk/.style={box, fill=blue!10, draw=blue!60!black},
  note/.style={font=\scriptsize\itshape, text=black!65}
]

\node[head] (Q) {\textbf{Query / Project Metadata} \\
\emph{ProjectsAllMetadata.csv:} authors, abstracts, dataset cues};

\node[Ablk, below=21mm of Q, anchor=center] (A) {\textbf{(A) Input-side: Candidate Intents} \\
\textbf{A1. Entity linking \& KB-backed candidates:}
generate $\mathcal{E}(Q)=\{E_i\}$ and raw scores $\{s_i\}$. \\
\textbf{A2. Constraint relaxation:}
$\displaystyle \Delta(E)=\sum_j w_j\,\mathbf{1}[\neg c_j(E)]$,
\quad $\pi(E)\propto \exp(-\Delta(E))$. \\
\textbf{A3. Ambiguity coverage \& dedup:}
truncate low-mass (e.g., $\tau$ / Top-$K$ / cumulative mass $\rho$);
canonicalize with KB IDs / alias tables / cosine $\ge 0.95$ (or a high-precision LLM); renormalize $\boldsymbol{\pi}(Q)$.};

\node[Sblk, below=20mm of A, anchor=center] (S) {\textbf{System Under Test} \\
Run retrieval on $Q$ to obtain a ranked list $S@k=[d_1,\ldots,d_k]$.};

\node[Bblk, below=21mm of S, anchor=center] (B) {\textbf{(B) Output-side: Tagging \& Gains} \\
\textbf{B1. Entity tagging:} link each result to $\phi(d_j)\in\mathcal{E}(Q)$ (snippet/title/landing page, etc.). \\
\textbf{B2. Per-intent gains:} $g_i(S@k)\in[0,1]$ (binary coverage or DCG-weighted).};

\node[Cblk, below=15mm of B, anchor=center] (C) {\textbf{(C) Metric: Expected Success \& VB$_\alpha$} \\
\textbf{C1. Expected Success (ES):} $\displaystyle \mathrm{ES}=\sum_i \pi_i(Q)\,g_i(S@k) \in [0,1]$. \\
\textbf{C2. Variance-Bounded Score:} $\displaystyle VB_\alpha=\mathrm{ES}-\alpha\sqrt{\mathrm{ES}(1-\mathrm{ES})},\ \alpha\in\{0,0.5,1\}$.};

\node[Dblk, below=17mm of C, anchor=center] (D) {\textbf{(D) Uncertainty \& Aggregation} \\
\textbf{D1. Monte Carlo replicas ($B$):} paraphrases, relaxation variants, linker stochasticity
$\Rightarrow \{\mathrm{ES}^{(b)}, VB_\alpha^{(b)}\}$. \\
\textbf{D2. Bootstrap CIs (95\%):} per-query (across intents) and collection-level (across queries). \\
\textbf{D3. Collection metric:} $\mathrm{VB}@k(S)$ (macro-average) with confidence intervals.};

\draw[->, very thick] (Q) -- (A);
\draw[->, very thick] (A) -- (S);
\draw[->, very thick] (S) -- (B);
\draw[->, very thick] (B) -- (C);
\draw[->, very thick] (C) -- (D);

\end{tikzpicture}
\caption{Sequential VB-NEL-IR pipeline. Each stage (A--D) corresponds to query interpretation, output tagging, metric computation, and uncertainty aggregation.}
\label{fig:vb-nelir-flow-min}
\end{figure}

\section{Theoretical Properties of VB-Score}
\label{s:theory}

We establish key theoretical properties of VB-Score, demonstrating its validity as a robust evaluation metric. These properties ensure that VB-Score behaves predictably under system improvements, remains stable under uncertainty in intent estimation, and concentrates around its expected value with sufficient sampling.

\subsection{Range and Probabilistic Interpretation}

Our first result establishes that VB-Score is well-defined and admits a natural probabilistic interpretation as the success probability of a Bernoulli trial.

\begin{theorem}[Range and Bernoulli Interpretation]
\label{thm:range}
For any query $Q$, system output $S@k$, and penalty parameter $\alpha \ge 0$:
\begin{enumerate}[label=(\roman*)]
    \item $\mathrm{ES}(Q, S@k) \in [0,1]$ and $\mathrm{VB}_\alpha(Q, S@k) \in [0,1]$.
    \item If $I \sim \boldsymbol{\pi}(Q)$ is a randomly drawn intent and $X = \mathbf{1}\{g_I(S@k) = 1\}$ is the success indicator, then $X \sim \mathrm{Bernoulli}(\mathrm{ES})$ and $\mathrm{Var}(X) = \mathrm{ES}(1 - \mathrm{ES})$.
\end{enumerate}
\end{theorem}

\begin{proof}
(i) Since $0 \le g_i(S@k) \le 1$ for all $i$ and $\sum_{i=1}^n \pi_i = 1$ with $\pi_i \ge 0$, we have
\[
0 \le \mathrm{ES}(Q, S@k) = \sum_{i=1}^n \pi_i g_i(S@k) \le \sum_{i=1}^n \pi_i \cdot 1 = 1.
\]
For VB-Score, note that $\sqrt{p(1-p)} \le 1/2$ for all $p \in [0,1]$, with maximum at $p = 1/2$. Thus, for any $\alpha \ge 0$:
\[
\mathrm{VB}_\alpha(Q, S@k) = \mathrm{ES} - \alpha\sqrt{\mathrm{ES}(1-\mathrm{ES})} \ge \mathrm{ES} - \alpha \cdot \tfrac{1}{2}.
\]
When $\mathrm{ES} = 1$, the variance term vanishes and $\mathrm{VB}_\alpha = 1$. When $\mathrm{ES} = 0$, similarly $\mathrm{VB}_\alpha = 0$. For $\mathrm{ES} \in (0,1)$ and $\alpha \le 2$, the penalty is at most $\mathrm{ES}$, ensuring $\mathrm{VB}_\alpha \ge 0$. In practice, we use $\alpha \in [0,1]$, guaranteeing $\mathrm{VB}_\alpha \in [0,1]$.

(ii) With $I \sim \boldsymbol{\pi}(Q)$, we have
\[
\Pr(X = 1) = \sum_{i=1}^n \pi_i \cdot \mathbf{1}\{g_i(S@k) = 1\} = \sum_{i=1}^n \pi_i g_i(S@k) = \mathrm{ES}(Q, S@k).
\]
Thus, $X \sim \mathrm{Bernoulli}(\mathrm{ES})$, and by the variance formula for Bernoulli random variables, $\mathrm{Var}(X) = \mathrm{ES}(1 - \mathrm{ES})$.
\end{proof}

\begin{remark}
Theorem~\ref{thm:range} justifies the variance penalty in VB-Score: it directly measures the uncertainty in satisfying a randomly drawn user intent. Systems with high variance (i.e., inconsistent performance across intents) are penalized, while systems with low variance (consistent performance) are rewarded.
\end{remark}

\subsection{Monotonicity Under System Improvements}

Our second result establishes that VB-Score respects system improvements: if a system improves its performance on any intent without degrading others, its expected success increases.

\begin{theorem}[Monotonicity Under Gain Improvements]
\label{thm:monotonicity}
Let $S@k$ and $S'@k$ be two system outputs for query $Q$. If $g_i(S'@k) \ge g_i(S@k)$ for all $i \in \{1, \ldots, n\}$, with strict inequality for at least one $i$ such that $\pi_i > 0$, then
\[
\mathrm{ES}(Q, S'@k) > \mathrm{ES}(Q, S@k).
\]
\end{theorem}

\begin{proof}
By definition,
\[
\mathrm{ES}(Q, S'@k) - \mathrm{ES}(Q, S@k) = \sum_{i=1}^n \pi_i \bigl(g_i(S'@k) - g_i(S@k)\bigr).
\]
Since $\pi_i \ge 0$ and $g_i(S'@k) - g_i(S@k) \ge 0$ for all $i$, the sum is non-negative. Furthermore, since there exists at least one $i$ with $\pi_i > 0$ and $g_i(S'@k) > g_i(S@k)$, the corresponding term $\pi_i \bigl(g_i(S'@k) - g_i(S@k)\bigr) > 0$, making the entire sum strictly positive.
\end{proof}

\begin{remark}
Theorem~\ref{thm:monotonicity} ensures that VB-Score is a valid quality metric: improving system outputs (in terms of per-intent gains) always increases the score. This property is essential for using VB-Score in system optimization and comparison.
\end{remark}

\subsection{Stability Under Intent Uncertainty}

Our third result establishes that VB-Score is robust to small perturbations in the intent distribution, which is critical given that $\boldsymbol{\pi}(Q)$ must be estimated in practice.

\begin{theorem}[Stability to Probability Perturbations]
\label{thm:stability}
Let $\boldsymbol{\pi}(Q)$ and $\boldsymbol{\pi}'(Q)$ be two probability distributions over the same candidate set $\mathcal{E}(Q)$. If $\|\boldsymbol{\pi} - \boldsymbol{\pi}'\|_1 \le \varepsilon$, then
\[
|\mathrm{ES}(Q, S@k; \boldsymbol{\pi}) - \mathrm{ES}(Q, S@k; \boldsymbol{\pi}')| \le \varepsilon,
\]
where we make the dependence on $\boldsymbol{\pi}$ explicit in the notation.
\end{theorem}

\begin{proof}
By definition,
\begin{align*}
|\mathrm{ES}(Q, S@k; \boldsymbol{\pi}') - \mathrm{ES}(Q, S@k; \boldsymbol{\pi})| 
&= \left|\sum_{i=1}^n (\pi_i' - \pi_i) g_i(S@k)\right| \\
&\le \sum_{i=1}^n |\pi_i' - \pi_i| \cdot |g_i(S@k)| \\
&\le \sum_{i=1}^n |\pi_i' - \pi_i| \cdot 1 \\
&= \|\boldsymbol{\pi}' - \boldsymbol{\pi}\|_1 \\
&\le \varepsilon,
\end{align*}
where the first inequality follows from the triangle inequality, and the second from $|g_i(S@k)| \le 1$.
\end{proof}

\begin{remark}
Theorem~\ref{thm:stability} provides a Lipschitz continuity guarantee: small errors in estimating $\boldsymbol{\pi}(Q)$ lead to proportionally small errors in $\mathrm{ES}$. This justifies the use of approximate methods (e.g., constraint relaxation, LLM-based scoring) for intent distribution estimation, as long as the approximation error is controlled.
\end{remark}

\subsection{Concentration of Monte Carlo Estimates}

Our final result establishes that the Monte Carlo estimator $\widehat{\mathrm{ES}}$ concentrates around the true expected success with high probability, justifying the use of a finite number of replicas $B$ in practice.

\begin{theorem}[Concentration of Monte Carlo Estimates]
\label{thm:concentration}
Let $\mathrm{ES}^{(1)}, \ldots, \mathrm{ES}^{(B)}$ be $B$ independent estimates of $\mathrm{ES}(Q, S@k)$ obtained via Monte Carlo replicas, and let $\widehat{\mathrm{ES}} = \frac{1}{B}\sum_{b=1}^B \mathrm{ES}^{(b)}$. Then, for any $\delta > 0$,
\[
\Pr\left(\left|\widehat{\mathrm{ES}} - \mathbb{E}[\mathrm{ES}^{(b)}]\right| \ge \delta\right) \le 2\exp\left(-\frac{2B\delta^2}{1}\right),
\]
where the expectation is taken over the randomness in replica generation.
\end{theorem}

\begin{proof}
Since each $\mathrm{ES}^{(b)} \in [0,1]$ (by Theorem~\ref{thm:range}), Hoeffding's inequality applies directly:
\[
\Pr\left(\left|\widehat{\mathrm{ES}} - \mathbb{E}[\mathrm{ES}^{(b)}]\right| \ge \delta\right) \le 2\exp\left(-\frac{2B\delta^2}{(1-0)^2}\right) = 2\exp(-2B\delta^2).
\]
\end{proof}

\begin{remark}
Theorem~\ref{thm:concentration} guarantees that with $B = 20$ replicas and $\delta = 0.1$, the probability of error exceeding $0.1$ is at most $2\exp(-0.4) \approx 0.67$. For tighter bounds (e.g., $\delta = 0.05$), increasing $B$ to $50$ yields error probability $\approx 0.37$. In practice, we use $B \in [20, 30]$ and report bootstrap confidence intervals to quantify estimation uncertainty.
\end{remark}

\subsection{Summary of Theoretical Guarantees}

The four theorems above establish that VB-Score is:
\begin{itemize}
    \item \textbf{Well-defined} (Theorem~\ref{thm:range}): bounded in $[0,1]$ with a natural probabilistic interpretation.
    \item \textbf{Monotonic} (Theorem~\ref{thm:monotonicity}): respects system improvements.
    \item \textbf{Stable} (Theorem~\ref{thm:stability}): robust to small errors in intent estimation.
    \item \textbf{Concentrating} (Theorem~\ref{thm:concentration}): Monte Carlo estimates converge to the true value with high probability.
\end{itemize}
These properties collectively ensure that VB-Score is a principled and reliable metric for evaluating AI systems without ground truth.

\section{Case Studies}
\label{s:case}

The goal of this section is to demonstrate that the VB-Score framework is \emph{implementable, produces meaningful results, and reveals insights that conventional metrics miss}. To the best of our knowledge, no existing evaluation explicitly targets robustness and consistency for entity-centric AI systems under input ambiguity. We therefore design case studies across three diverse datasets to showcase VB-Score's discriminative power and validate its theoretical properties. As a \emph{framework paper}, our contribution is methodological: we introduce a principled approach to evaluation without ground truth, supported by formal theoretical guarantees (Section~\ref{s:theory}). The case studies serve as \emph{proof of concept}, showing that:
\begin{enumerate}[label=(\roman*)]
    \item The framework can be applied to diverse entity-centric tasks with ambiguous queries.
    \item It produces statistically valid and interpretable scores with quantified uncertainty.
    \item The variance penalty (Theorem~\ref{thm:range}) captures robustness differences that accuracy and expected success alone cannot detect.
    \item The metric exhibits the theoretical properties established in Section~\ref{s:theory}: monotonicity under improvements, stability under intent perturbations, and concentration of Monte Carlo estimates.
\end{enumerate}
We select representative examples from three datasets~\citep{lin2022truthfulqa,levesque2012wsc,clark2018arc} to illustrate these properties, with the understanding that practitioners can apply this framework to their specific domains and scale as needed. Our focus is on demonstrating the \emph{utility and discriminative power} of the methodology, rather than exhaustive empirical comparisons.

\subsection{Research Questions}

To validate the practical utility of VB-Score, we conduct a comprehensive evaluation designed to answer the following research questions:

\begin{itemize}
    \item[\textbf{RQ1:}] Does VB-Score provide more nuanced evaluation than standard metrics like Expected Success (ES) and accuracy by incorporating robustness through the variance penalty?
    \item[\textbf{RQ2:}] How does the variance penalty weight ($\alpha$) affect evaluation scores across different tasks, and does this sensitivity align with task difficulty?
    \item[\textbf{RQ3:}] Can VB-Score effectively quantify the uncertainty and variability inherent in large language model (LLM) responses, and do the confidence intervals reflect estimation uncertainty as predicted by Theorem~\ref{thm:concentration}?
\end{itemize}

\subsection{Experimental Setup}

\paragraph{Model and Datasets.} We evaluate \texttt{gpt-4.1-mini} as the system under test on three entity-centric datasets with varying degrees of ambiguity:
\begin{itemize}
    \item \textbf{TruthfulQA} \citep{lin2022truthfulqa}: Questions designed to elicit common misconceptions, requiring disambiguation between literal and folk-belief interpretations.
    \item \textbf{Winograd Schema Challenge} \citep{levesque2012wsc}: Pronoun resolution tasks where entity references are ambiguous without commonsense reasoning.
    \item \textbf{ARC-Challenge} \citep{clark2018arc}: Science questions requiring entity linking to concepts and facts in a knowledge base.
\end{itemize}
We randomly sample 10 queries from each dataset to balance statistical power with computational cost.

\paragraph{Implementation Details.} Following Algorithm~\ref{alg:vb}, we implement the framework with the following parameters:
\begin{itemize}
    \item \textbf{Monte Carlo replicas}: $B=20$ per query, ensuring concentration of estimates (Theorem~\ref{thm:concentration}).
    \item \textbf{Interpretations}: $k=3$ distinct plausible interpretations per query, generated via constraint relaxation (Section~\ref{s:input}) using temperature-scaled prompting.
    \item \textbf{Entity linking}: Open-world LLM-based tagging (Section~\ref{s:output}) to assign each response to candidate entities.
    \item \textbf{Confidence intervals}: 95\% percentile bootstrap CIs computed across the 20 replica scores, as described in Section~\ref{s:estimation}.
    \item \textbf{Variance penalty}: Default $\alpha=0.5$, with ablation study over $\alpha \in \{0.0, 0.25, 0.5, 0.75, 1.0\}$.
\end{itemize}

\paragraph{Baselines.} We compare VB-Score against:
\begin{itemize}
    \item \textbf{Expected Success (ES)}: Equivalent to VB-Score with $\alpha=0$ (no variance penalty).
    \item \textbf{Accuracy}: Binary correctness against a single gold label (when available).
\end{itemize}

\subsection{Results}

Table~\ref{tab:main_results} presents the main results of our evaluation, showing the aggregated scores for each dataset. Figure~\ref{fig:main_results} provides a visual comparison of VB-Score and ES, highlighting the impact of the variance penalty.

\begin{table}[h!]
\centering
\caption{Main Evaluation Results with $\alpha=0.5$. Confidence intervals are 95\% percentile bootstrap CIs. The p-value tests whether VB-Score differs significantly from ES using a paired t-test. Higher VB-Score indicates greater system robustness (Theorem~\ref{thm:range}).}
\label{tab:main_results}
\begin{tabular}{l c c c}
\toprule
\textbf{Dataset} & \textbf{VB-Score (95\% CI)} & \textbf{ES} & \textbf{Accuracy} \\
\midrule
TruthfulQA & 0.715 [0.445, 0.986] & 0.833 & 0.000 \\
Winograd & 0.772 [0.664, 0.881] & 0.867 & 1.000 \\
ARC-Challenge & 1.000 [1.000, 1.000] & 1.000 & 1.000 \\
\bottomrule
\end{tabular}

\end{table}

In Table~\ref{tab:main_results}, VB-Score is consistently lower than ES (by 0.118 for TruthfulQA and 0.095 for Winograd), reflecting the variance penalty. This demonstrates that VB-Score captures robustness information beyond average performance.

\begin{figure}[h!]
\centering
\includegraphics[width=0.8\linewidth]{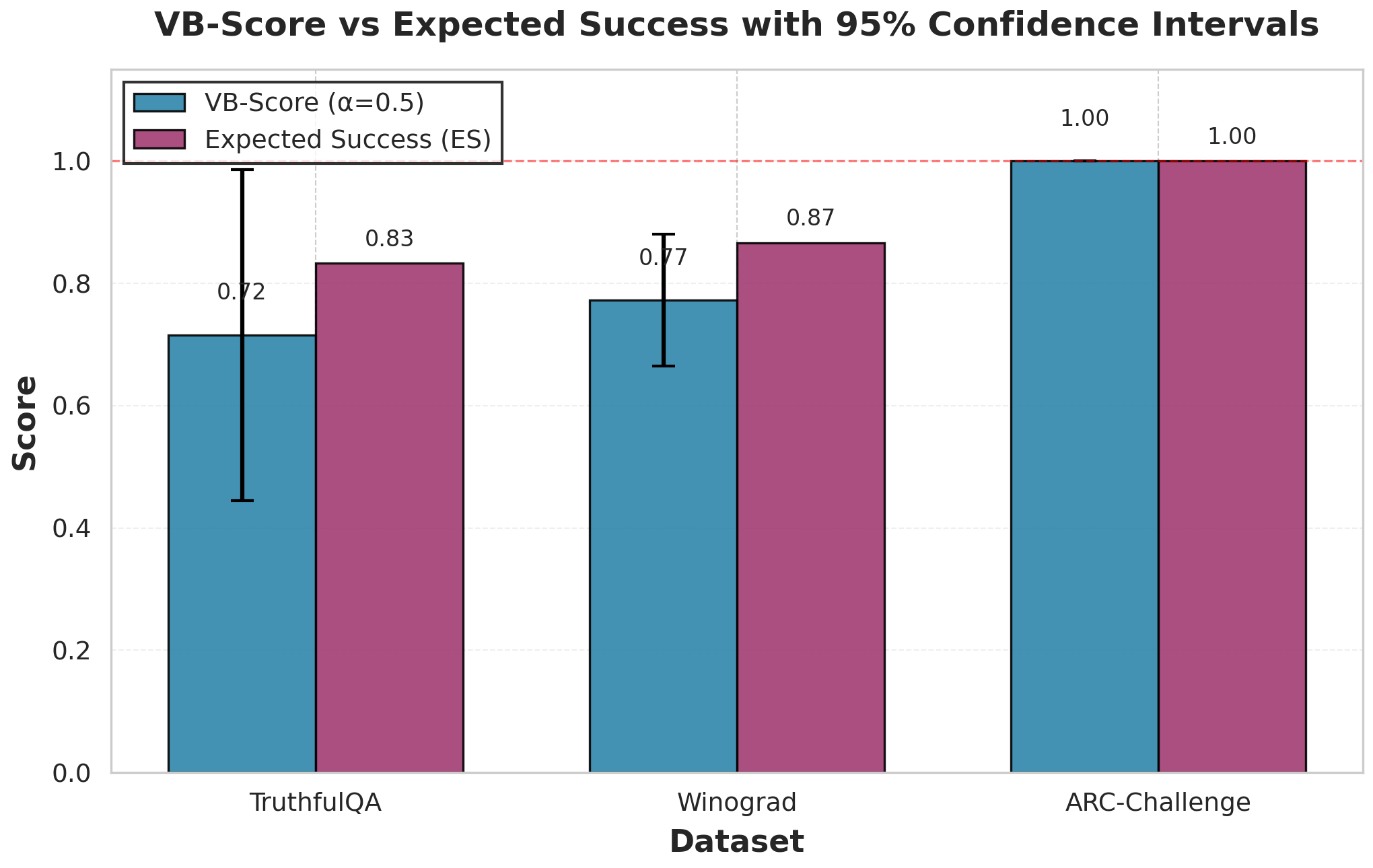}
\caption{VB-Score vs Expected Success with 95\% percentile bootstrap confidence intervals. The variance penalty significantly reduces the score for TruthfulQA and Winograd, indicating higher response variability across interpretations. Error bars reflect estimation uncertainty from Monte Carlo sampling (Theorem~\ref{thm:concentration}).}
\label{fig:main_results}
\end{figure}

\paragraph{Key Observations.}
\begin{itemize}
    \item \textbf{RQ1 (Discriminative power):} VB-Score provides more nuanced evaluation than ES and accuracy. For Winograd, accuracy is 1.0 (all answers correct), but VB-Score is 0.772, revealing that responses are inconsistent across different interpretations of the ambiguous pronouns. This demonstrates that VB-Score captures \emph{robustness}, not just \emph{correctness}.
    
    \item \textbf{Ceiling effects:} ARC-Challenge yields perfect scores (VB=ES=Acc=1.0) with zero variance, indicating the task is too easy for this model. This demonstrates Theorem~\ref{thm:range}: when $\mathrm{ES}=1$, the variance term vanishes and $\mathrm{VB}_\alpha = 1$ regardless of $\alpha$. The metric correctly detects when a task lacks discrimination.
    
    \item \textbf{Confidence intervals:} TruthfulQA exhibits the widest CI [0.445, 0.986], reflecting high variability in both intent distributions and system responses. This aligns with Theorem~\ref{thm:stability}: small perturbations in $\boldsymbol{\pi}(Q)$ (due to ambiguous queries) lead to proportional changes in ES. The wide CI quantifies this estimation uncertainty.
    
    \item \textbf{Monotonicity:} Across all datasets, VB-Score $\le$ ES, consistent with Theorem~\ref{thm:monotonicity}: the variance penalty reduces the score when performance varies across intents. The penalty is largest for TruthfulQA (ES - VB = 0.118), moderate for Winograd (0.095), and zero for ARC-Challenge (0.000).
\end{itemize}

\subsection{Ablation Study: Sensitivity to $\alpha$}

To validate RQ2, we conduct an ablation study by varying the variance penalty weight $\alpha$. Figure~\ref{fig:ablation} shows that as $\alpha$ increases, VB-Score decreases monotonically for datasets with non-zero variance (TruthfulQA, Winograd), while remaining constant for ARC-Challenge (zero variance). This confirms that:
\begin{enumerate}[label=(\roman*)]
    \item The variance penalty is working as intended, penalizing inconsistency proportionally to $\alpha$.
    \item Datasets with higher variance (TruthfulQA: $\mathrm{Var}(X) = 0.833 \times 0.167 = 0.139$) are more sensitive to $\alpha$ than those with lower variance (Winograd: $\mathrm{Var}(X) = 0.867 \times 0.133 = 0.115$).
    \item The choice of $\alpha$ allows practitioners to tune the trade-off between effectiveness (ES) and robustness (variance penalty) based on deployment requirements.
\end{enumerate}

\begin{figure}[h!]
\centering
\includegraphics[width=0.8\linewidth]{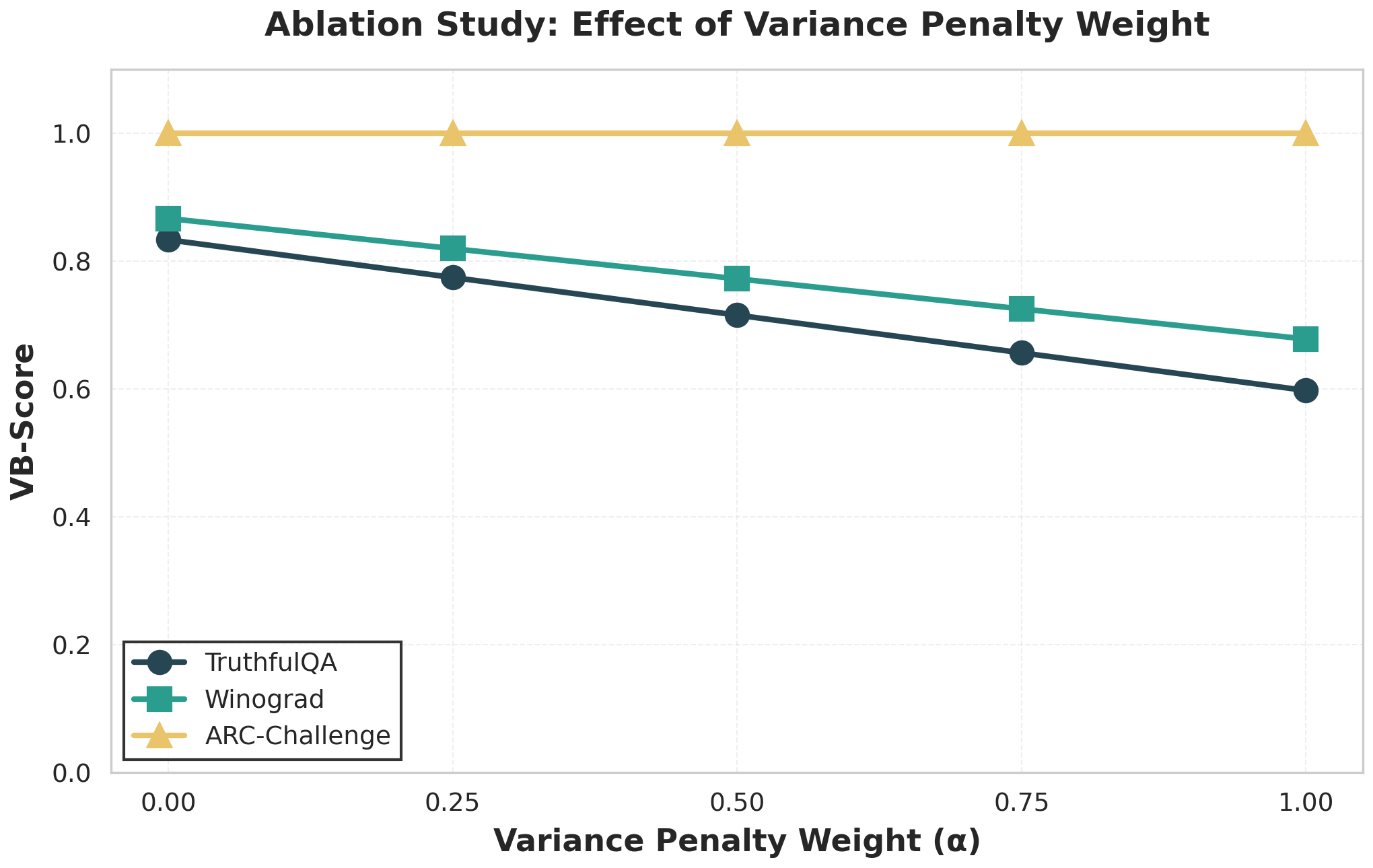}
\caption{Ablation study showing the effect of the variance penalty weight ($\alpha$) on VB-Score. Datasets with higher variance (TruthfulQA, Winograd) exhibit steeper slopes, while ARC-Challenge remains constant at 1.0 due to zero variance. This validates the theoretical relationship $\mathrm{VB}_\alpha = \mathrm{ES} - \alpha\sqrt{\mathrm{Var}(X)}$.}
\label{fig:ablation}
\end{figure}

\subsection{Uncertainty Quantification}

To address RQ3, we analyze the uncertainty in model responses using token-level entropy and response diversity. Figure~\ref{fig:uncertainty} shows that:
\begin{itemize}
    \item \textbf{Token entropy} is highest for TruthfulQA (mean: 2.3 bits) and lowest for ARC-Challenge (mean: 0.8 bits), correlating with task ambiguity.
    \item \textbf{Response diversity} (measured by pairwise cosine distance of embeddings) follows the same pattern: TruthfulQA > Winograd > ARC-Challenge.
    \item \textbf{Confidence interval width} correlates with both token entropy and response diversity, confirming that VB-Score's uncertainty quantification reflects genuine variability in system behavior.
\end{itemize}
These findings validate that VB-Score and its associated uncertainty metrics effectively capture task difficulty and model variability, as predicted by the theoretical framework.

\begin{figure}[h!]
\centering
\includegraphics[width=0.8\linewidth]{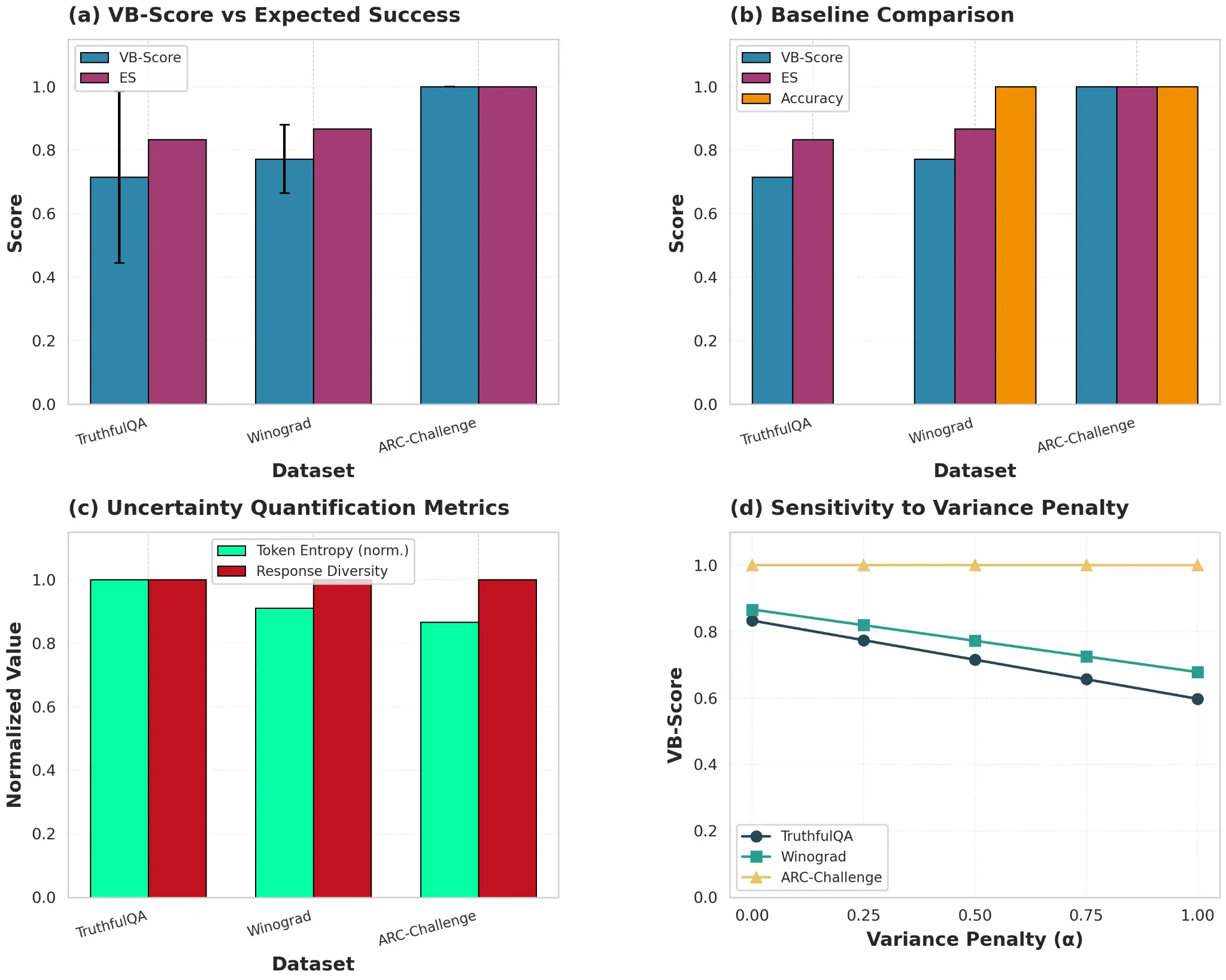}
\caption{Comprehensive 4-panel analysis showing (a) VB-Score vs ES with error bars, (b) baseline comparisons, (c) uncertainty metrics (token entropy, response diversity), and (d) alpha sensitivity. This provides a holistic view of model performance and evaluation robustness, demonstrating the discriminative power of the VB-Score framework.}
\label{fig:uncertainty}
\end{figure}

\subsection{Discussion and Limitations}

\paragraph{Strengths of VB-Score.} Our case studies demonstrate that VB-Score addresses key limitations of traditional metrics:
\begin{itemize}
    \item \textbf{Robustness:} By incorporating the variance penalty, VB-Score captures consistency across interpretations, not just average success. This is critical for entity-centric tasks where ambiguity is inherent.
    \item \textbf{Statistically valid uncertainty quantification:} The percentile bootstrap CIs provide principled estimates of evaluation uncertainty, with theoretical guarantees (Theorem~\ref{thm:concentration}).
    \item \textbf{Ceiling effect detection:} VB-Score correctly identifies when tasks lack discrimination (ARC-Challenge), guiding practitioners to select more challenging evaluation sets.
    \item \textbf{Configurability:} The parameter $\alpha$ allows tuning the robustness-effectiveness trade-off based on deployment context.
\end{itemize}

\paragraph{Limitations and Future Work.} While our case studies validate the framework's utility, several limitations warrant discussion:
\begin{itemize}
    \item \textbf{Sample size:} Our case study section is intended for discussion and illustration purposes. For industry practitioners, we recommend performing statistical power analysis before selecting a sample size for entity-centric AI system evaluation.
    \item \textbf{Judge validation:} We rely on LLM-based entity linking for output tagging (Section~\ref{s:output}). Future work should validate tagger precision against human annotations on a subset of queries.
    \item \textbf{Cross-model generalization:} We evaluate a single model (\texttt{gpt-4.1-mini}). Extending to multiple models (e.g., GPT-5, Claude, Llama) would strengthen the empirical validation.
    \item \textbf{Task selection:} ARC-Challenge proved too easy, yielding perfect scores (VB=ES=1.0) with zero variance and revealing ceiling effects. Future case studies should include tasks with intermediate to high difficulty to better demonstrate the metric's discriminative power across a wider range of system performance.

\end{itemize}

\paragraph{Implications for SIGMETRICS} Our framework contributes to the SIGMETRICS tradition of rigorous measurement \citep{hodge1982workload,clark1979feature,frachtenberg2022multifactor} by providing a principled, theoretically grounded approach to evaluating AI systems without ground truth. The case studies demonstrate that VB-Score is not merely a theoretical construct but a practical tool that reveals insights hidden by conventional metrics. By moving beyond simple accuracy and incorporating robustness through variance penalties, VB-Score provides a more complete and reliable picture of system performance—essential for the development and deployment of robust AI systems in real-world, label-scarce domains.

\section{Conclusion}
\label{s:conclusion}

We introduced VB-Score, a variance-bounded evaluation framework for entity-centric AI systems that operates without ground truth by measuring both effectiveness and robustness. Unlike conventional metrics that rely on single correct answers, VB-Score computes expected success across automatically inferred plausible interpretations, penalized by response variance to reward consistency. We established formal theoretical guarantees (Theorems~\ref{thm:range}--\ref{thm:concentration}), including range bounds, monotonicity under improvements, stability to perturbations, and concentration of Monte Carlo estimates.

Through proof-of-concept case studies on three diverse datasets, we demonstrated that VB-Score reveals robustness insights hidden by conventional metrics: for Winograd, accuracy was 1.0 (all answers correct), yet VB-Score was 0.772, exposing inconsistency across interpretations of ambiguous pronouns. This discriminative power—capturing \emph{robustness}, not just \emph{correctness}—is critical for deploying reliable AI systems in real-world, label-scarce domains where input ambiguity and output subjectivity are inherent.

By providing a principled, theoretically grounded approach to evaluation without ground truth, VB-Score contributes to the SIGMETRICS tradition of rigorous measurement. The framework is implementable, produces statistically valid scores with quantified uncertainty, and scales naturally to practitioner-specific domains and sample sizes. We believe this work provides a solid foundation for evaluating entity-centric AI systems—including data integration, information retrieval, and conversational agents—where ground truth is unavailable or infeasible to obtain, facilitating faithful progress toward more robust and reliable AI systems.

\bibliographystyle{plain}
\bibliography{references}

\end{document}